\documentclass[11pt]{article}
\usepackage[margin=0.85in]{geometry}
\usepackage{amsmath,amssymb,amsthm}
\usepackage{graphicx}
\usepackage{tikz}
\usepackage[most]{tcolorbox}
\usepackage{enumitem}
\usepackage{hyperref}
\usetikzlibrary{calc}

\newtheorem{theorem}{Theorem}[section]
\newtheorem{lemma}[theorem]{Lemma}

\theoremstyle{definition}

\theoremstyle{remark}

\newtcolorbox{modelbox}{
  colback=gray!7,
  colframe=gray!35,
  boxrule=0.6pt,
  arc=1mm,
  left=8pt,
  right=8pt,
  top=7pt,
  bottom=7pt
}

\newtcolorbox{algorithmbox}[1][]{
  colback=gray!4,
  colframe=gray!25,
  colbacktitle=gray!10,
  coltitle=black,
  boxrule=0.5pt,
  arc=1mm,
  left=7pt,
  right=7pt,
  top=6pt,
  bottom=6pt,
  before skip=8pt,
  after skip=8pt,
  title=#1,
  fonttitle=\bfseries
}
\tcolorboxenvironment{theorem}{
  colback=gray!5,
  colframe=gray!25,
  boxrule=0.5pt,
  arc=1mm,
  left=6pt,
  right=6pt,
  top=5pt,
  bottom=5pt,
  before skip=10pt,
  after skip=10pt
}

\usetikzlibrary{arrows.meta,calc}

\DeclareMathOperator{\vol}{vol}

\title{Sorting from Counterexamples}

\author{
Noga Alon\thanks{
Princeton University, Princeton, NJ 08544, USA, and
Tel Aviv University, Tel Aviv 69978, Israel.
Email: \texttt{nalon@math.princeton.edu}.
Research supported in part by NSF grant DMS-2553988.
}
\and
Shay Moran\thanks{
Technion -- Israel Institute of Technology and Google Research.
Email: \texttt{smoran@technion.ac.il}.
Research supported by PBC-VATAT, by the Technion Center for Machine Learning
and Intelligent Systems (MLIS), by ISF grant 2362/26, and by the European
Union (ERC, GENERALIZATION, 101039692).
Views and opinions expressed are however those of the authors only and do not
necessarily reflect those of the European Union or the European Research
Council Executive Agency. Neither the European Union nor the granting
authority can be held responsible for them.
}
\and
Shlomo Moran\thanks{
Technion -- Israel Institute of Technology.
Email: \texttt{moran@cs.technion.ac.il}.
}
}
\date{}

\begin{document}

\maketitle

\begin{abstract}
Consider the following problem of learning an unknown linear order on $n$
items. In each round, the learner guesses a complete ordering of the items and
receives either confirmation that the guess is correct or a counterexample:
a pair of items in the wrong order. The goal is to identify the unknown
order using as few queries as possible. We study this problem when up to $k$
of the returned counterexamples may be untruthful, where $k$ is not known in
advance. We determine the optimal query complexity up to
constant factors:
\[
    \Theta(n\log n + nk).
\]
Thus, while the noiseless complexity matches the classical complexity of
sorting, each untruthful counterexample incurs an additional cost of order
$n$. The upper bound is based on a geometric representation of permutations
and Gr\"unbaum's theorem, while the lower bound combines sorting arguments
with a Condorcet-type construction. We also study the case where the target ranking has a low-dimensional geometric representation: each item is represented by a point in
$\mathbb{R}^d$, and the ranking is obtained by projecting the points onto
an unknown direction. For these classes we give an upper bound of
$O(d^2\log n+dk)$ and a lower bound of $\Omega(d\log n+dk)$, leaving a
factor of $d$ gap in the noiseless term.
\end{abstract}
\section{Introduction}

Imagine asking a recommendation system where to go for dinner with your
partner. It returns a ranked list of nearby restaurants, placing a sushi
restaurant first and a falafel place near the bottom. What kind of feedback can such a system use to improve? Empirical studies have found pairwise preference feedback to be fast and consistent~\cite{CarteretteEtAl2008,ShahEtAl2014}. Such feedback can be much easier to provide than a full reranking and may require little active attention from the user: they can simply point out a pair that seems badly misordered: the falafel restaurant is
buried near the bottom even though they strongly prefer it, while the sushi
restaurant appears first despite being much less appealing. 
Such a pair is
a natural counterexample because its relative order is badly wrong in the
proposed ranking. 
This simple interaction motivates the problem we study:
learning an unknown ranking from counterexamples.
\medskip
\begin{modelbox}
\textbf{Ranking from counterexamples.}
There is an unknown target ranking $\pi^\star$ over $n$ items. On each round:

\begin{enumerate}[leftmargin=22pt,itemsep=2pt,topsep=4pt]
    \item The learner proposes a linear order $\pi$.
    \item The oracle either declares that $\pi$ is correct, in which case
    $\pi=\pi^\star$ and the interaction terminates, or returns an unordered
    pair $\{i,j\}$, indicating that the relative order of $i$
    and $j$ in~$\pi$ is wrong.
\end{enumerate}
A returned pair is \emph{truthful} if $\pi$ and $\pi^\star$ indeed disagree on the relative order of $i$ and $j$, and \emph{untruthful} otherwise. We allow at most $k$ untruthful counterexamples throughout the interaction, where $k$ is not known to the learner. Declarations that the proposed ranking is correct are always truthful.
\end{modelbox}
Thus, even if the learner proposes $\pi^\star$, the oracle may conceal this
fact by returning a counterexample, but such a counterexample is necessarily
untruthful and therefore consumes one of the at most $k$ allowed lies.
For a randomized learner, we measure the worst-case expected number of
queries, where the expectation is over the learner's internal randomness.
\medskip

This model can be viewed as a variation on the classical sorting problem. In
comparison-based sorting, the algorithm chooses a pair of items and learns
their relative order. Here, the learner instead proposes a complete ranking,
and the feedback reveals one pair on which this proposal is allegedly
incorrect. Our goal is to understand how many such rounds are needed to
identify the target ranking.

Allowing untruthful counterexamples is also natural in the motivating
example. A ranking may represent the typical preferences of a population of
similar users, while feedback from an individual user may occasionally
disagree with it. There are richer ways to model such variability, for
example through probabilistic preference models such as the
Bradley--Terry model \cite{BradleyTerry1952}. In this work, we focus on a
mathematically simpler formulation: the total number of untruthful
counterexamples is bounded by $k$, where $k$ is not known to the learner,
and we make no further assumptions on when these counterexamples occur or
how they are chosen. As we will see, even this basic setting already
presents nontrivial challenges. 
A useful feature of our algorithms is that they do not use $k$ as an input:
the same learner works for every finite lie budget. This parameter-free
viewpoint also makes the framework amenable to richer noise models, although
we leave a systematic study of such models for future work.


\medskip

At first sight, classical learning algorithms such as halving or weighted
majority appear well suited to this problem. One can maintain a weight on
every possible target ranking, initially assigning all rankings the same
weight. Whenever a counterexample $\{i,j\}$ is returned, the weights of
rankings that order $i$ and $j$ as in the learner's proposal are decreased
by a constant factor. The learner can then orient every pair according to the
weighted majority of the candidate rankings. This approach is naturally
robust to untruthful feedback: the total weight decreases whenever a
counterexample is returned, whereas the weight of the target ranking is
decreased only when the feedback is untruthful.

There is, however, a basic obstacle: the weighted-majority relation need not
itself be a linear order. This is the familiar Condorcet phenomenon
\cite{Black1958}.\footnote{The Marquis de Condorcet (1743--1794) was a
French mathematician and political philosopher who studied collective
decision making and voting~\cite{Condorcet1785}. The phenomenon bearing his name shows that
pairwise majority preferences can be cyclic even when each individual
voter's preferences form a linear order.} For example, consider the three
rankings
\[
    a \succ b \succ c,
    \qquad
    b \succ c \succ a,
    \qquad
    c \succ a \succ b.
\]
A majority ranks $a$ above $b$, $b$ above $c$, and $c$ above $a$, producing
a directed cycle. Thus, although weighted majority gives a natural prediction
for every pair separately, these predictions need not combine into a valid
ranking.
\begin{figure}[t]
    \centering
    \begin{tikzpicture}[>=stealth,scale=1.05]
        \node[circle,draw,minimum size=8mm] (a) at (90:1.55) {$a$};
        \node[circle,draw,minimum size=8mm] (b) at (210:1.55) {$b$};
        \node[circle,draw,minimum size=8mm] (c) at (330:1.55) {$c$};

        \draw[->,thick] (a) to[bend right=10] (b);
        \draw[->,thick] (b) to[bend right=10] (c);
        \draw[->,thick] (c) to[bend right=10] (a);
    \end{tikzpicture}

    \caption{The pairwise majority relation may contain a Condorcet cycle,
    even though each individual ranking is a linear order.}
    \label{fig:condorcet}
\end{figure}
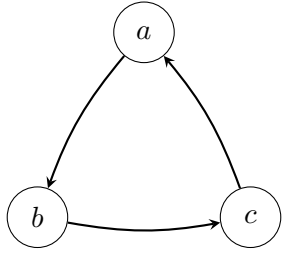
This makes \emph{properness} central to the problem: on every round, the
learner must output a genuine linear order rather than an arbitrary
collection of pairwise predictions. The restriction is also natural in the motivating example: what the user seeks
from the recommendation system is a ranked list of restaurants, rather than
an arbitrary collection of pairwise recommendations. 
If improper predictions were allowed, the classical weighted-majority
algorithm \cite{LittlestoneWarmuth1994} would give
\[
    O(\log(n!)+k)=O(n\log n+k)
\]
rounds, as reviewed in the proof overview below. As we show, requiring the learner to remain proper changes the
picture: the optimal query complexity is
\[
    \Theta(n\log n+nk).
\]
In particular, the cost of each untruthful counterexample increases by a
factor of $n$. Our upper bound is obtained by a geometric modification of
weighted majority, which allows us to maintain proper predictions using
Gr\"unbaum's Theorem \cite{Grunbaum1960}, as discussed in detail below.

\subsection{Our Results}
We study this problem in two settings. We first consider the unrestricted
setting, where the target may be any ranking of the $n$ items. We then turn
to ranking classes with additional geometric structure. 
Such classes arise naturally when rankings are governed by a small number of
features. In the restaurant example, a restaurant might be represented by
features such as price, distance, cuisine, and ambience, while a user's
preferences specify how much weight to place on each feature. Ranking the
restaurants by the resulting scores gives a low-dimensional geometric
ranking. More generally, items may be ranked by projection along an unknown
direction or by distance from an unknown point.

Our first result gives a tight characterization of the query complexity in
the unrestricted setting.
\begin{theorem}
For an arbitrary target ranking over $n$ items, in the presence of at most
$k$ untruthful counterexamples, the optimal number of queries is
\[
    \Theta(n\log n + nk).
\]
Moreover, the upper bound is achieved by a deterministic learner that
does not know $k$.
\end{theorem}

The upper bound can be viewed as a geometric way of making weighted majority
proper. The basic geometric picture is easiest to see when all
counterexamples are truthful. Represent each ranking by a cell of the cube
$[0,1]^n$: a point $x=(x_1,\ldots,x_n)$ induces the ranking obtained by
sorting its coordinates, and the $n!$ possible rankings correspond to
$n!$ equal-volume cells. A comparison between two items corresponds to a
linear inequality between two coordinates, and therefore each
counterexample restricts the possible points to a halfspace. Consequently,
the points consistent with all counterexamples seen so far form a convex
body.

The learner queries the ranking induced by the center of gravity of this
body. The key geometric ingredient is Gr\"unbaum's theorem, which says that a halfspace containing the center of gravity of a convex
body must contain a constant fraction of its volume. Thus, if the proposed
ranking is wrong, the returned counterexample eliminates a constant fraction
of the remaining volume, and hence a constant fraction of the remaining
rankings. This geometric viewpoint on linear orders goes back to Stanley
\cite{Stanley1981} and was developed further by Kahn and Saks
\cite{KahnSaks1984}. We present this truthful case in the proof overview as
a warm-up and explain the geometric argument there in a self-contained
manner.

When counterexamples may be untruthful, this argument cannot be applied
directly: permanently discarding all rankings inconsistent with a
counterexample might discard the target itself. We therefore replace hard
elimination by a smooth analogue of the weighted-majority update. This
creates a non-uniform distribution over the cube rather than a uniform
distribution over a convex body. The update is chosen so that this
distribution remains log-concave, allowing us to use the more general,
log-concave form of Gr\"unbaum's theorem. In this way we retain constant
progress in the total weight while ensuring that truthful counterexamples
have only a small effect on the weight near the target; an untruthful
counterexample may have a larger effect, but only by a factor of order $n$.
This yields the $O(n\log n+nk)$ upper bound. The proof overview develops
these ideas, including the required form of Gr\"unbaum's theorem, from
scratch.

The two terms in the lower bound arise from different constructions. The
$\Omega(n\log n)$ term already appears in the noiseless setting and follows
from a connection with balanced comparison trees for sorting. The
$\Omega(nk)$ term follows from a simple Condorcet-cycle construction. Thus,
the same phenomenon that prevents pairwise majority from being proper also
plays a role in showing that untruthful counterexamples are inherently more
costly for proper learners.

\paragraph{Geometric rankings.}
Our second result concerns ranking classes with additional geometric
structure. Consider $n$ items represented by vectors
$v_1,\ldots,v_n\in\mathbb R^d$. Every direction $w\in\mathbb R^d$
induces a ranking by ordering the items according to
\[
    \langle v_1,w\rangle,\ldots,\langle v_n,w\rangle.
\]
We call a class of rankings \emph{$d$-dimensional geometric} if it can be
represented in this way.

\begin{figure}[t]
\centering
\begin{tikzpicture}[
    >=Latex,
    every node/.style={font=\small}
]

\definecolor{cblue}{RGB}{45,95,190}
\definecolor{cred}{RGB}{205,55,45}
\definecolor{cgreen}{RGB}{45,155,70}
\definecolor{cpurple}{RGB}{135,75,190}

\draw[->, gray!55, line width=0.7pt]
    (0,0) -- (8.3,0) node[right, black] {$e_1$};
\draw[->, gray!55, line width=0.7pt]
    (0,0) -- (0,4.6) node[above, black] {$e_2$};

\begin{scope}[shift={(0.7,0.55)}, rotate=27]

    \draw[->, line width=1.1pt]
        (0,0) -- (7.1,0)
        node[right] {$w$};

    \coordinate (p3) at (0.9,1.10);
    \coordinate (q3) at (0.9,0);
    \draw[dashed, cgreen, line width=0.8pt] (p3) -- (q3);
    \fill[cgreen] (p3) circle (2.7pt);
    \fill[cgreen] (q3) circle (2.3pt);
    \node[above left=1pt] at (p3) {$v_3$};

    \coordinate (p1) at (2.85,1.85);
    \coordinate (q1) at (2.85,0);
    \draw[dashed, cblue, line width=0.8pt] (p1) -- (q1);
    \fill[cblue] (p1) circle (2.7pt);
    \fill[cblue] (q1) circle (2.3pt);
    \node[above left=1pt] at (p1) {$v_1$};

    \coordinate (p4) at (4.45,-1.15);
    \coordinate (q4) at (4.45,0);
    \draw[dashed, cpurple, line width=0.8pt] (p4) -- (q4);
    \fill[cpurple] (p4) circle (2.7pt);
    \fill[cpurple] (q4) circle (2.3pt);
    \node[below right=1pt] at (p4) {$v_4$};

    \coordinate (p2) at (5.85,1.55);
    \coordinate (q2) at (5.85,0);
    \draw[dashed, cred, line width=0.8pt] (p2) -- (q2);
    \fill[cred] (p2) circle (2.7pt);
    \fill[cred] (q2) circle (2.3pt);
    \node[above right=1pt] at (p2) {$v_2$};

\end{scope}

\node[anchor=east] at (1.3,-1.05) {\textbf{Projections onto $w$}};

\draw[->, line width=0.9pt]
    (1.55,-1.05) -- (8.1,-1.05)
    node[right] {$w$};

\fill[cgreen]  (2.3,-1.05) circle (2.5pt);
\fill[cblue]   (3.8,-1.05) circle (2.5pt);
\fill[cpurple] (5.3,-1.05) circle (2.5pt);
\fill[cred]    (6.8,-1.05) circle (2.5pt);

\node[cgreen, below=3pt]  at (2.3,-1.05) {$v_3$};
\node[cblue, below=3pt]   at (3.8,-1.05) {$v_1$};
\node[cpurple, below=3pt] at (5.3,-1.05) {$v_4$};
\node[cred, below=3pt]    at (6.8,-1.05) {$v_2$};

\node[anchor=east] at (1.3,-1.9) {\textbf{Induced ranking}};

\node at (4.85,-1.9) {
    {\color{cgreen}$v_3$}
    $\;\prec\;$
    {\color{cblue}$v_1$}
    $\;\prec\;$
    {\color{cpurple}$v_4$}
    $\;\prec\;$
    {\color{cred}$v_2$}
};

\end{tikzpicture}

\caption{
A geometric ranking in two dimensions.
The items are represented by vectors
$v_1,\ldots,v_n\in\mathbb{R}^2$.
A direction $w$ induces a ranking by ordering the items according to
their projections onto $w$.
Dashed lines indicate orthogonal projections.
}
\label{fig:geometric-ranking}
\end{figure}
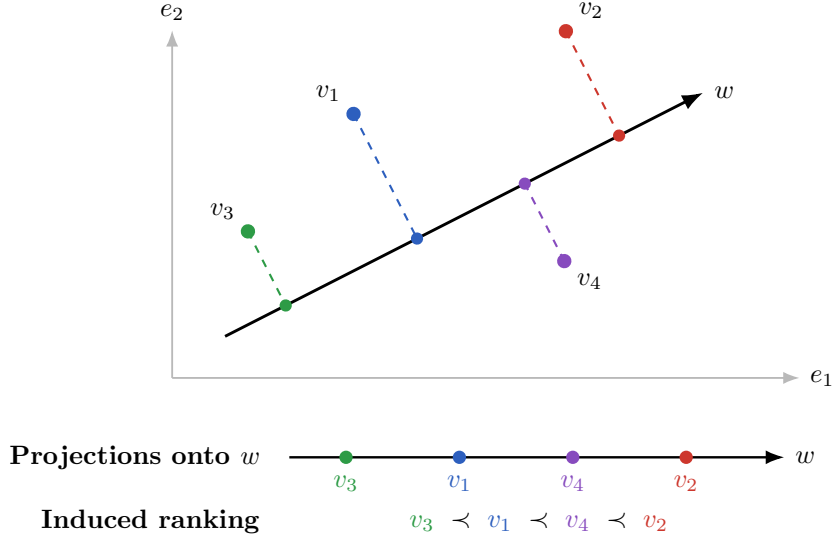
We emphasize that the learner's queries are still required to be genuine
linear orders, but they need not themselves belong to the geometric class.
In other words, the learner may propose an ordering that is not induced by
any direction $w$ in the given representation.

The geometric dimension can be much smaller than the number of items, and it
is therefore natural to ask whether the query complexity can depend on $d$
rather than directly on $n$.
\begin{theorem}
Every $d$-dimensional geometric class of rankings over $n$ items can be
learned using
\[
    O(d^2\log n+dk)
\]
queries in the presence of at most $k$ untruthful counterexamples.
Conversely, there are $d$-dimensional geometric classes for which every
learner whose queries are linear orders requires
\[
    \Omega(d\log n+dk)
\]
queries.
\end{theorem}

The upper bound follows from a different way of rounding a weighted majority
of linear orders to a linear order. Ordinary majority may contain cycles, but for
$d$-dimensional geometric rankings a sufficiently strong majority relation
is acyclic. In particular, a threshold of roughly $1-1/d$ produces a
relation that can be extended to a linear order. This allows us to use
weighted majority more directly, at the cost of an additional factor of $d$
in the progress made on each truthful round.
Moreover, the upper bound can be achieved while choosing every query to be
a ranking induced by some direction in the same geometric representation.
Thus, when the class consists of all rankings induced by the representation,
the learner can be proper in the usual sense. This may also be useful in
feature-based applications. In the restaurant example, such a ranking is
specified by only $d$ coefficients, rather than by an arbitrary ordering
of the $n$ restaurants. This gives a compact representation that can also be
applied to new items and may offer some interpretability through the relative
importance assigned to the different features.

The dependence on $k$ in the query complexity is tight, while a factor of $d$ remains
between the upper and lower bounds in the noiseless term. This leaves the
following natural question open:
\[
    \text{Can the } O(d^2\log n) \text{ upper bound be improved to }
    O(d\log n)\text{?}
\]

\paragraph{Computational aspects.}
Our focus in this work is on query complexity, rather than on obtaining
fully efficient implementations of the learning algorithms. In particular,
our upper bound for arbitrary rankings is formulated in terms of barycenters
of log-concave distributions, which we do not compute explicitly. The
distributions arising in our algorithm are, however, efficiently evaluable
log-concave distributions, and standard random-walk methods allow one to
sample approximately from such distributions in randomized polynomial time
\cite{LovaszVempala2006,LovaszVempala2007}. Averaging sufficiently many
samples provides a natural way to approximate the barycenter; related
random-walk methods for replacing exact centroids by empirical
approximations appear, for example, in \cite{BertsimasVempala2004}.
This suggests that our approach should admit a randomized polynomial-time
implementation, although we do not work out the required approximation
guarantees here. We leave a systematic study of efficient algorithms for
future work. In particular, it would be interesting to find efficient
combinatorial algorithms for this problem. Besides avoiding general-purpose
sampling machinery, such algorithms might also lead to more direct proofs
of our upper bounds that circumvent the use of Gr\"unbaum's theorem.

\paragraph{Organization of the paper.}
We begin with a proof overview in Section~\ref{sec:proof-overview}, which
develops the main ideas behind our results and explains the geometric and
combinatorial ingredients used in the proofs. We then discuss related work
in Section~\ref{sec:related-work}, followed by the full arguments, first for
arbitrary rankings in Section~\ref{sec:arbitrary-rankings} and then for
geometric ranking classes in Section~\ref{sec:geometric-rankings}. We
conclude with several open questions and directions for future research in
Section~\ref{sec:open-questions}.

\section{Proof Overview}
\label{sec:proof-overview}

The goal of this section is to explain the main ideas behind our proofs,
focusing on the parts that are less routine and on the intuition behind the
different constructions. Besides providing an overview of the proofs, we hope
that this discussion will make the more detailed arguments in the following
sections easier to follow. We first discuss the upper bounds, and then turn to
the lower bounds.

\subsection{Upper Bounds}

\paragraph{Weighted majority as a benchmark.}
It is useful to begin by recalling the classical weighted-majority algorithm
\cite{LittlestoneWarmuth1994}. Think of each possible target ranking as an
expert, and initially assign weight one to every expert.

\begin{algorithmbox}[Weighted majority (improper benchmark)]
On each round:
\begin{enumerate}
    \item For every unordered pair $\{i,j\}$, let $A_t(i\prec j)$ denote
    the total current weight of rankings that place $i$ before $j$.
    Predict $i\prec j$ if
    \(A_t(i\prec j)> A_t(j\prec i)\),
    and predict $j\prec i$ if  \(A_t(j\prec i)> A_t(i\prec j)\), breaking ties arbitrarily.
    These pairwise predictions constitute the output of this benchmark;
    they are not required to be transitive and therefore need not define
    a linear order.

    \item If a counterexample $\{i,j\}$ is returned, multiply by $1/2$
    the weight of every ranking that agrees with the learner's prediction
    on this pair.
\end{enumerate}
\end{algorithmbox}

The standard analysis is simple.
Let $W_t$ denote the total weight after $t$ returned counterexamples;
initially, $W_0=n!$. Whenever a counterexample is
returned, at least half of the total weight is multiplied by $1/2$, and hence
the total weight decreases by a factor of at most $3/4$. On the other hand,
the weight of the target ranking is decreased only when the counterexample is
untruthful. Thus, after $T$ rounds and at most $k$ untruthful
counterexamples,
\[
    2^{-k}
    \leq W_T
    \leq n!\left(\frac34\right)^T,
\]
and therefore
\[
    T=O(\log(n!)+k)=O(n\log n+k).
\]

\paragraph{Arbitrary rankings: making majority geometric.}
We first consider the unrestricted class of all rankings. There are two main
ideas in the proof. The first is to replace pairwise majority by a geometric
notion of averaging that always produces a linear order. The second is to
modify the weighted-majority update so that this geometric averaging continues
to work in the presence of untruthful counterexamples. 

We begin with the first idea, in the simpler setting where all
counterexamples are truthful. In this case there is no need to maintain
weights: once a ranking is contradicted by a truthful counterexample, it can
be discarded permanently. Thus, instead of weighted majority, we may work
with the simpler halving algorithm and maintain only the current version
space $\mathcal V_t$ of rankings consistent with all counterexamples seen so
far.

Define the majority relation on the items by putting an edge $i\to j$
whenever a strict majority of the rankings in $\mathcal V_t$ place $i$
before $j$. If this relation were acyclic, we could extend it to a linear
order and query that order. Any returned counterexample $\{i,j\}$ would then
disagree with a majority edge, and therefore eliminate at least half of the
rankings in $\mathcal V_t$. Starting from $n!$ possible rankings, this would
identify the target after at most $\lceil \log_2(n!)\rceil$ rounds.

The difficulty is that the majority relation need not be acyclic: it may
contain a Condorcet cycle. Our first geometric idea replaces this majority
relation by a proper linear order while retaining constant-factor progress.

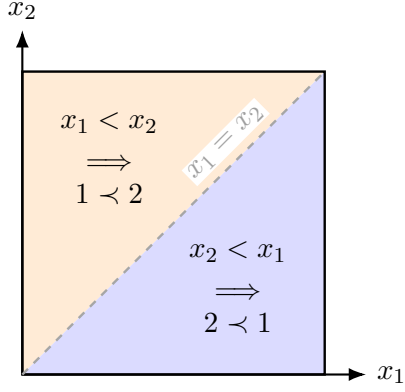
\begin{figure}[t]
    \centering
    \begin{tikzpicture}[x=1cm,y=1cm,>=Latex]

        \fill[orange!16]
            (0,0) -- (0,4) -- (4,4) -- cycle;

        \fill[blue!14]
            (0,0) -- (4,0) -- (4,4) -- cycle;

        \draw[line width=0.9pt] (0,0) rectangle (4,4);

        \draw[->,line width=0.8pt]
            (0,0) -- (4.55,0)
            node[right] {$x_1$};

        \draw[->,line width=0.8pt]
            (0,0) -- (0,4.55)
            node[above] {$x_2$};

        \draw[dashed,gray!70,line width=0.9pt]
            (0,0) -- (4,4)
            node[pos=0.72,sloped,above=3pt,
                 text=gray!75,fill=white,inner sep=1pt]
            {$x_1=x_2$};

        \node[align=center] at (1.15,2.85) {
            $x_1<x_2$\\[2pt]
            $\Longrightarrow$\\[-1pt]
            $1\prec 2$
        };

        \node[align=center] at (2.85,1.15) {
            $x_2<x_1$\\[2pt]
            $\Longrightarrow$\\[-1pt]
            $2\prec 1$
        };

    \end{tikzpicture}

    \caption{
    The geometric representation of rankings for $n=2$.
    A point $x\in[0,1]^2$ induces a ranking according to the order of its
    coordinates. The line $x_1=x_2$ splits the square into the two
    cells corresponding to the two possible linear orders.
    }
    \label{fig:ranking-cells}
\end{figure}

To obtain a proper analogue of halving, we use the geometric viewpoint on
linear extensions initiated by Stanley \cite{Stanley1981} and developed
further by Kahn and Saks \cite{KahnSaks1984}. Consider the cube
$[0,1]^n$. A point $x=(x_1,\ldots,x_n)$ with distinct coordinates induces
a ranking by sorting its coordinates: $i$ is ranked before $j$ whenever
$x_i<x_j$. The complement of the hyperplanes $x_i=x_j$ therefore consists of $n!$ open cells, one for each linear order, and all these cells have the same volume $1/n!$.
Figure~\ref{fig:ranking-cells} depicts this representation in the simplest
case $n=2$, where the diagonal $x_1=x_2$ separates the two possible
rankings. Moreover, every comparison constraint is
a halfspace constraint: for example, the requirement that $i$ precedes $j$
corresponds to
\[
    \bigl\{x\in[0,1]^n : x_j-x_i>0\bigr\}.
\]
Consequently, the set $K_t$ of points whose induced rankings are consistent
with all counterexamples seen so far is a convex body, obtained by
intersecting the cube with halfspaces. Since every ranking cell has volume
$1/n!$,
\[
    \operatorname{vol}(K_t)=\frac{|\mathcal V_t|}{n!}.
\]
\begin{algorithmbox}[Centroid halving]
Maintain the convex body $K_t\subseteq[0,1]^n$ consisting of all points
consistent with the counterexamples seen so far. On each round, let $c_t$ be
the centroid of $K_t$ and query the ranking induced by the coordinates of
$c_t$.
\end{algorithmbox}
The reason this query makes constant progress is
Gr\"unbaum's theorem \cite{Grunbaum1960}: if $K\subseteq\mathbb R^n$ is a
convex body with centroid $c$, then every halfspace $H$ containing $c$
satisfies
\[
    \frac{\operatorname{vol}(K\cap H)}
         {\operatorname{vol}(K)}
    \geq
    \left(\frac{n}{n+1}\right)^n
    \geq \frac1e.
\]
Suppose, for example, that the centroid ranks $i$ before $j$, and the oracle
returns $\{i,j\}$ as a counterexample. The halfspace corresponding to the
centroid's orientation contains the centroid, and is therefore at least a
$1/e$ fraction of the current volume. Since this is precisely the side
eliminated by the counterexample, every round removes at least a $1/e$
fraction of the remaining volume. Hence
\[
    \operatorname{vol}(K_{t+1})
    \leq
    \left(1-\frac1e\right)\operatorname{vol}(K_t),
\]
and again $O(\log(n!))=O(n\log n)$ rounds suffice.

Thus, in the truthful setting, Gr\"unbaum's theorem plays the role of
halving: we lose the factor $1/2$, replacing it by another universal
constant, but gain the important property that the prediction is always a
linear order. This argument is closely connected to the classical
Kahn--Saks theorem on balanced comparisons in linear extensions
\cite{KahnSaks1984}; in particular, the geometric ingredient needed for our
truthful upper bound already appears in their work and in the subsequent papers \cite{KahnLinial1991} and \cite{KarzanovKhachiyan1991}.

\paragraph{Allowing untruthful counterexamples.}
With untruthful counterexamples, we can no longer permanently discard a
ranking when it becomes inconsistent with the feedback, so we replace hard
halving by a weighted update. The obvious attempt is to lift weighted majority to
the cube by assigning to every ranking cell the weight of the corresponding
ranking. Unfortunately, the resulting piecewise-constant density is generally
not log-concave, and hence the geometric argument above no longer applies.

Our second main idea is to replace this discontinuous update by a convex
surrogate. We maintain a density over the cube rather than a discrete weight
on each ranking.
\begin{algorithmbox}[Log-concave weighted majority]
Maintain a log-concave density $w_t$ on $[0,1]^n$. On each round, let $c_t$
be its barycenter and query the ranking induced by the coordinates of $c_t$.

Suppose the returned counterexample $\{i,j\}$ indicates that $i$ should be
ranked before $j$. Thus, the queried ranking places $j$ before $i$. Define
\[
    z_t(x):=x_i-x_j,
    \qquad
    \ell_t(x):=\bigl(1+\gamma n z_t(x)\bigr)_+,
\]
where $(a)_+:=\max\{a,0\}$. Since the query places $j$ before $i$, we have
$z_t(c_t)\geq0$. Update
\[
    w_{t+1}(x)=w_t(x)e^{-\ell_t(x)},
\]
where $\gamma>0$ is a sufficiently large universal constant.
\end{algorithmbox}
The particular form of this update is chosen for two reasons. First,
$(1+\gamma n z_t(x))_+$ is convex, and hence the update preserves log-concavity. We may therefore use the log-concave version of Gr\"unbaum's
inequality \cite{LovaszVempala2007}: any halfspace containing the barycenter
of a log-concave probability measure has probability at least $1/e$.
Since $\ell_t\geq 1$ throughout the halfspace $\{z_t\geq0\}$, a constant
fraction of the total weight receives a constant multiplicative penalty.
Thus, just as in ordinary weighted majority, the total weight decreases by
a constant factor on every round.

Second, this update does not decrease the weight near the target ranking too
quickly. Consider the cell of the cube corresponding to the true ranking.
On a truthful round this entire cell lies on the correct side of the
comparison hyperplane, and the surrogate loss is nonzero only in a thin strip
of width $O(1/n)$ near that hyperplane. For a uniformly random point in a
ranking cell, the gaps between consecutive coordinates are typically of order
$1/n$; choosing the constant $\gamma$ sufficiently large makes the expected
surrogate loss on a truthful round arbitrarily small. In contrast, an
untruthful counterexample may incur loss $O(n)$ on the target cell.

The resulting comparison of masses mirrors the elementary weighted-majority
analysis. After $T$ rounds, the total mass is at most
\(\rho^T\)
for some universal constant $\rho<1$, while the mass inside the target cell
is at least, up to constants in the exponent,
\[
    \frac{1}{n!}
    \exp\left(-O(T/\gamma)-O(nk)\right).
\]
Choosing $\gamma$ sufficiently large and comparing these two bounds gives
\[
    T=O(\log(n!)+nk)=O(n\log n+nk).
\]
One minor subtlety in the formal argument is that the density inside the
target cell is no longer uniform after the first update. We therefore
analyze its total mass relative to the uniform measure on the target cell and
use Jensen's inequality; this is where the detailed proof differs slightly
from the informal calculation above.

\paragraph{Geometric ranking classes: making strong majority proper.}
For $d$-dimensional geometric ranking classes, we can stay much closer to the
original weighted-majority algorithm. We keep an ordinary weight on every
candidate ranking and, after a counterexample, multiply by $1/2$ the weights
of the candidates that agree with the learner on the returned pair. The only
question is again how to turn the resulting pairwise weighted preferences
into a linear order.

Here geometric structure gives a different solution. Instead of keeping every
pair supported by a simple majority, we keep only sufficiently strong
majorities: roughly speaking, we orient $i$ before $j$ only when a
$1-O(1/d)$ fraction of the current weight ranks $i$ before $j$. The key
geometric fact is that this strong-majority relation is acyclic, and hence
can be extended to a linear order.

To see the reason, represent the items by vectors
$v_1,\ldots,v_n\in\mathbb R^d$. A comparison $i\prec j$ corresponds to a
linear inequality on the direction $w$,
\[
    \langle v_j-v_i,w\rangle>0.
\]
If the strong-majority relation contained a directed cycle, the corresponding
difference vectors would have a positive dependence whose sum is zero.
By Carath\'eodory's theorem, such a dependence can already be witnessed by
only $O(d)$ of these vectors. No direction $w$ can satisfy all the
corresponding inequalities simultaneously. Therefore every candidate ranking
must disagree with at least one of these $O(d)$ comparisons. But if each of
them were supported by all but an $O(1/d)$ fraction of the total weight, the
union bound would give a contradiction. This is exactly what makes the
strong-majority relation acyclic.

We query any linear extension of this relation. For every pair oriented by
the resulting ranking, at least an $\Omega(1/d)$ fraction of the total weight
supports that orientation; otherwise the reverse orientation would itself be
a strong-majority edge. Consequently, whenever a counterexample is returned,
an $\Omega(1/d)$ fraction of the total weight is multiplied by $1/2$.
The total weight therefore decreases by a factor
\[
    1-\Omega(1/d)
\]
on every round, while the target weight is decreased only on the at most
$k$ untruthful rounds.

Finally, an arrangement of the $\binom n2$ comparison hyperplanes in
$\mathbb R^d$ induces at most $n^{O(d)}$ distinct rankings, and hence the
initial logarithmic weight is $O(d\log n)$. The usual weighted-majority
calculation now gives
\[
    O\bigl(d\cdot(d\log n+k)\bigr)
    =
    O(d^2\log n+dk)
\]
queries.

There are also alternative geometric proofs of this upper bound, based on
Helly's theorem or on the centerpoint theorem~\cite{Matousek2002}. These have
the additional feature that the queried linear order can be chosen to be
induced by a direction in the given representation; see
Section~\ref{sec:geometric-upper}.

\subsection{Lower Bounds}

The lower bounds are simpler and more direct than the upper
bounds. We begin with two combinatorial arguments for unrestricted rankings,
giving $\Omega(n\log n)$ and $\Omega(nk)$. We then extend these ideas to
geometric rankings and use an additional polynomial construction to obtain
the full $\Omega(d\log n+dk)$ lower bound.

For simplicity, throughout this overview we describe the lower bounds using
an adversary that does not commit to the target ranking $\pi^\star$ in
advance. Instead, the adversary maintains a set of possible targets and
chooses its counterexamples so that this set remains nonempty, committing to
a target only at the end of the interaction. This is slightly stronger than
our formal model, where $\pi^\star$ is fixed from the outset, but makes the
arguments more transparent. In the formal proofs, we recover the same lower
bounds with a target chosen in advance, using a randomized construction.

\paragraph{Arbitrary rankings.}
The $\Omega(n\log n)$ bound follows directly from sorting. Fix a balanced
comparison tree for sorting $n$ items, in which every root-to-leaf path has
length $\Omega(n\log n)$; for example, MergeSort induces such a tree. The
adversary follows a path in this tree. At the current node, suppose the tree
compares $i$ and $j$. The learner's proposed ranking fixes an orientation of
this pair. The adversary returns $\{i,j\}$ as a counterexample and follows
the child corresponding to the opposite orientation. At the end of the
interaction, the ranking at the resulting leaf is consistent with all the
counterexamples that were returned. Thus the interaction is forced to
traverse an entire root-to-leaf path, requiring~$\Omega(n\log n)$ queries.

The $\Omega(nk)$ bound comes from essentially the same Condorcet cycle that
obstructs pairwise majority. Consider the directed cycle
\[
    1\to2\to\cdots\to n\to1,
\]
and, for each $r\in[n]$, let $\sigma_r$ be the linear order obtained by
breaking the cycle at its $r$-th edge. Every linear order violates at least
one edge of the cycle. Whenever the learner proposes an order, the adversary
returns such a violated edge as a counterexample. If edge $r$ is returned,
this counterexample is truthful for every candidate $\sigma_s$ except
$\sigma_r$. Consequently, $\sigma_r$ remains a possible target as long as
edge $r$ has been returned at most $k$ times. To rule out all but one of the
$n$ candidates, the adversary must therefore return at least $k+1$ copies of
at least $n-1$ different edges. This gives
\[
    \Omega(nk)
\]
queries. Together, the two arguments yield a lower bound of
\[
    \Omega(n\log n+nk)
\]
on the number of queries.

\paragraph{Geometric rankings.}
Both lower bounds above already apply to geometric rankings when the number
of items equals the dimension. Indeed, represent $d$ items by the standard
basis vectors $e_1,\ldots,e_d\in\mathbb R^d$. Since
$\langle e_i,w\rangle=w_i$, choosing the coordinates of $w$ in any prescribed
order realizes any desired linear order on the $d$ items. The unrestricted
lower bounds therefore immediately give
\[
    \Omega(d\log d+dk).
\]

It remains to obtain the dependence on $n$ in the noiseless term. For this
we use polynomial rankings. Identify the items with the points
$1,\ldots,n$ on the real line, and let a polynomial $p$ of degree at most
$d$ rank them according to
\[
    p(1),p(2),\ldots,p(n).
\]
These rankings are $d$-dimensional geometric: the constant term does not
affect the ranking, and
\[
    p(i)
    =
    a_0+\left\langle
        (a_1,\ldots,a_d),(i,i^2,\ldots,i^d)
    \right\rangle.
\]
Thus, such polynomial orders are $d$-dimensional geometric with the items represeneted by points on the moment curve $\{(x,x^2,\ldots,x^d) : x\in\mathbb{R}\}$.

We restrict attention to polynomials with $\Theta(d)$ turning points, with
one turning point in each of $\Theta(d)$ disjoint intervals, each containing
$\Theta(n/d)$ consecutive items; see
Figure~\ref{fig:polynomial-lower-bound}. The locations of these turning
points can be chosen essentially independently by prescribing the roots of
$p'$.

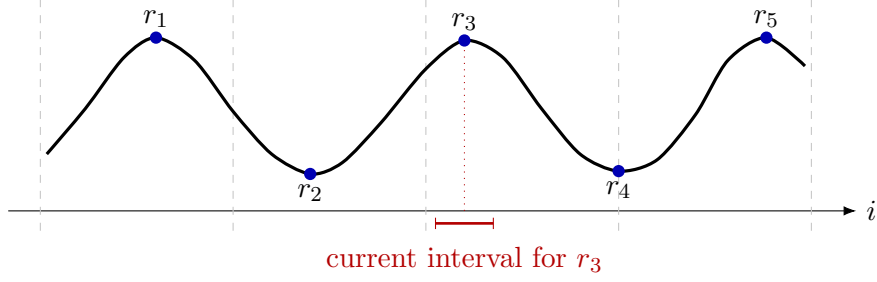
\begin{figure}[t]
    \centering
    \begin{tikzpicture}[x=0.85cm,y=0.75cm,>=Latex]

        \draw[->] (0,0) -- (13.2,0) node[right] {$i$};

        \foreach \x in {0.5,3.5,6.5,9.5,12.5} {
            \draw[dashed,gray!45] (\x,-0.35) -- (\x,3.8);
        }

        \draw[very thick,smooth]
            plot coordinates {
                (0.6,1.0)
                (1.2,1.8)
                (1.8,2.7)
                (2.3,3.05)
                (2.9,2.65)
                (3.5,1.75)
                (4.1,1.0)
                (4.7,0.65)
                (5.2,0.85)
                (5.8,1.55)
                (6.5,2.5)
                (7.1,3.0)
                (7.7,2.7)
                (8.3,1.8)
                (8.9,1.0)
                (9.5,0.7)
                (10.1,0.9)
                (10.7,1.7)
                (11.2,2.65)
                (11.8,3.05)
                (12.4,2.55)
            };

        \fill[blue!70!black] (2.3,3.05) circle (2.3pt);
        \fill[blue!70!black] (4.7,0.65) circle (2.3pt);
        \fill[blue!70!black] (7.1,3.0) circle (2.3pt);
        \fill[blue!70!black] (9.5,0.7) circle (2.3pt);
        \fill[blue!70!black] (11.8,3.05) circle (2.3pt);

        \node[above] at (2.3,3.05) {$r_1$};
        \node[below] at (4.7,0.65) {$r_2$};
        \node[above] at (7.1,3.0) {$r_3$};
        \node[below] at (9.5,0.7) {$r_4$};
        \node[above] at (11.8,3.05) {$r_5$};

        \draw[red!70!black,line width=1pt]
            (6.65,-0.22) -- (7.55,-0.22);

        \draw[red!70!black]
            (6.65,-0.32) -- (6.65,-0.12);
        \draw[red!70!black]
            (7.55,-0.32) -- (7.55,-0.12);

        \node[red!70!black,below=5pt] at (7.1,-0.22)
            {current interval for $r_3$};

        \draw[red!70!black,dotted]
            (7.1,0) -- (7.1,2.9);

    \end{tikzpicture}

    \caption{
    Schematic view of the polynomial construction.
    The items are consecutive points on the real line, and the polynomial
    has one turning point in each of $\Theta(d)$ disjoint intervals.
    A comparison of two consecutive items near the midpoint of the current
    uncertainty interval determines on which side the corresponding turning
    point may lie, yielding a binary-search-type lower bound.
    }
    \label{fig:polynomial-lower-bound}
\end{figure}

Consider one of these intervals. As long as the corresponding turning point
is only known to lie in some subinterval, consider two consecutive items near
its midpoint. Their relative order determines on which side of this location
the turning point lies. The adversary returns this pair with the orientation
opposite to the learner's proposed ranking, thereby retaining a realizable
choice of the turning point in one of the two halves. Thus each
counterexample reveals at most one step of a binary search, and locating a
single turning point requires
\[
    \Omega(\log(n/d))
\]
queries.

The adversary can repeat this argument separately for each of the
$\Theta(d)$ turning points. Hence
\[
    \Omega\bigl(d\log(n/d)\bigr)
\]
queries are necessary.

Combining this with the previous $\Omega(d\log d)$ bound gives
\[
    \Omega(d\log n),
\]
since
\[
    \max\!\left\{
        d\log d,\,
        d\log(n/d)
    \right\}
    \geq
    \frac12 d\log n.
\]
Together with the $\Omega(dk)$ bound, this yields
\[
    \Omega(d\log n+dk).
\]

\section{Related Work}\label{sec:related-work}
\paragraph{Learning and ranking from comparisons.}
Pairwise comparisons have a long history as a means of eliciting preferences
and judgments, with classical statistical models including those of Thurstone
and Bradley--Terry \cite{Thurstone1927,BradleyTerry1952}. They have also been
studied extensively as a form of supervision in theoretical learning.
Balcan, Vitercik, and White \cite{BalcanVitercikWhite2016} study learning
real-valued combinatorial functions from pairwise comparisons, while
Kane et al.~\cite{KaneEtAl2017} study comparison queries in active
classification. Related and subsequent work studies, among other questions,
robustness to noisy comparisons and the power of comparisons for learning
linear classifiers
\cite{XuEtAl2017,HopkinsKaneLovett2020,HopkinsEtAl2020}.

A closely related literature studies learning rankings from pairwise
comparisons, under both adaptive and non-adaptive sampling schemes
\cite{JamiesonNowak2011,Ailon2011,AilonBegleiterEzra2012,
RajkumarAgarwal2016,AgarwalEtAl2017}. Of particular relevance to our geometric setting, Jamieson and
Nowak~\cite{JamiesonNowak2011} consider objects embedded in a
low-dimensional Euclidean space, with rankings determined by distance from
a common reference point, and show that this structure can substantially
reduce the number of required comparisons.
Related forms of relative supervision also arise in contrastive learning
\cite{AlonEtAl2024}.

Our model differs in the direction of the interaction. Rather than choosing
a pair and receiving its relative order, the learner proposes a complete
ranking and the oracle returns a pair witnessing an error in that ranking.
In particular, every learner query must itself be a linear order, making
properness a central issue in our setting.

Another line of work studies ranking and preference learning from pairwise
comparison data. Balcan et al.~\cite{BalcanEtAl2007} give a robust reduction from AUC
ranking to binary classification, showing how guarantees for the induced
pairwise classification problem translate into guarantees for ranking. More recently, Pukdee, Balcan, and
Ravikumar~\cite{PukdeeBalcanRavikumar2026} study preference learning from
pairwise comparison data through the Bradley--Terry model, including the
question of what is recovered when the model is misspecified. These works consider substantially different learning settings from ours,
but share the perspective of pairwise preferences as a basic form of
supervision.

\paragraph{Learning from equivalence queries.}
Our interaction is closely related to the classical model of learning from
equivalence queries, developed by Angluin
\cite{Angluin1987,Angluin1988}. In that model, the learner proposes a
hypothesis and, if it is incorrect, receives a counterexample witnessing a
disagreement with the target.
A substantial literature studies this model and variants of it; see, for
example, \cite{Vaandrager2017} for a broader perspective on active model
learning.

More recently, Braverman et al.~\cite{BravermanEtAl2026} revisit
equivalence-query learning from the perspective of proper learning and
counterexample generation. They study less adversarial, symmetric
counterexample generators, including random counterexamples previously
considered in
\cite{AngluinDohrn2017,Bhatia2021,ChaseFreitagReyzin2024},
and also consider a bandit-feedback variant in which the counterexample is
returned without its correct label.

Our model can be viewed as a structured equivalence-query problem in which
the hypotheses are rankings and counterexamples are pairs of misordered
items. There are, however, two important differences. First, we allow a
bounded number of untruthful counterexamples. Second, the structural
requirement that every learner query be a linear order is central to our
results; in the geometric setting, the queried order need not itself belong
to the target geometric class.

\paragraph{Proper learning and projection.}
Our geometric upper bound is also related to techniques developed for
proper learning. Kane et al.~\cite{KaneLivniMoranYehudayoff2019} use a
majority-vote construction together with a mechanism for replacing it by a
hypothesis from the prescribed class. Bousquet et al.~\cite{BousquetEtAl2020}
formalize this idea through the \emph{projection number}, which quantifies
when the predictions on which a majority has sufficiently large margin can
be simultaneously realized by a hypothesis from the class. Related questions about the price of restricting
online learners to proper or otherwise simple predictors are studied by
Hanneke, Livni, and Moran~\cite{HannekeLivniMoran2021}.

The same principle appears in our geometric upper bound. We first form
pairwise predictions supported by an overwhelming weighted majority and
then show, using the geometry of the ranking class, that these strong
predictions can be extended to a valid linear order. The settings and analyses are different---in particular, the projection
arguments above are used in realizable learning, whereas we apply the same
general principle in the presence of boundedly many untruthful
counterexamples---but the underlying idea is closely related.

\paragraph{Sorting under partial information and linear extensions.}
The use of convex geometry in the study of linear extensions goes back to
Stanley \cite{Stanley1981}, who applied the Aleksandrov--Fenchel
inequalities to obtain log-concavity results for linear extensions of
posets. Kahn and Saks \cite{KahnSaks1984} subsequently proved that every
non-total partial order contains a pair whose two possible orientations
each occur in a constant fraction of its linear extensions. More precisely,
they obtained the constant $3/11$. As a consequence, the
information-theoretic lower bound for sorting under partial information is
tight up to a constant factor. Simpler geometric proofs were
later found by Kahn and Linial \cite{KahnLinial1991} and by Karzanov and
Khachiyan \cite{KarzanovKhachiyan1991}. Both are based on the Brunn--Minkowski
theorem and give the weaker constants $1/(2e)$ in \cite{KahnLinial1991} and $1/e^2$ in \cite{KarzanovKhachiyan1991}.

This line of work is particularly relevant to our upper bound. The
geometric representation of linear extensions by regions in the cube
\cite{Stanley1986}, and the use of convex-geometric inequalities to find
balanced pairs, are closely related to the argument we use in the truthful
setting. Our treatment with untruthful counterexamples requires going
beyond the uniform measure on the set of linear extensions: the algorithm
maintains a non-uniform log-concave density and uses the log-concave form
of Gr\"unbaum's theorem.

\paragraph{Sorting with noisy comparisons.}
There is also an extensive literature on sorting and related comparison
problems when individual comparison outcomes may be erroneous
\cite{RivestEtAl1980,FeigeEtAl1994,BravermanMossel2008}.
The feedback model studied here is different: the learner does not choose a
pair to compare, but instead proposes an entire ranking and the oracle
chooses which pair to return as a counterexample. In particular, the main
difficulty in our setting is not only robustness to noise, but also the
requirement that the learner's pairwise predictions form a coherent linear
order.

\section{Arbitrary Rankings}
\label{sec:arbitrary-rankings}

In this section we prove the upper and lower bounds for the unrestricted
class of all rankings. We begin with the upper bound.

\subsection{The Upper Bound}
\label{sec:arbitrary-upper}

We first recall the geometric representation used in the proof overview.
For a ranking $\pi$ of $[n]$, let
\[
    \mathcal C_\pi
    :=
    \left\{
        x\in(0,1)^n:
        x_{\pi(1)}<x_{\pi(2)}<\cdots<x_{\pi(n)}
    \right\}.
\]
We refer to $\mathcal C_\pi$ as the \emph{cell} corresponding to $\pi$.
Up to their measure-zero boundaries, the $n!$ cells partition the cube
$[0,1]^n$, and by symmetry they all have volume $1/n!$.

This representation is closely related to the convex-geometric approach to
linear extensions initiated by Stanley \cite{Stanley1981}. Kahn and Saks
\cite{KahnSaks1984} used this approach to prove that every non-total partial
order contains a pair whose two possible orientations each occur in a
constant fraction of its linear extensions. Kahn and Linial
\cite{KahnLinial1991} later gave a Brunn--Minkowski proof of the same
conclusion, with a slightly weaker constant; a similar independent proof was given by Karzanov and Khachiyan \cite{KarzanovKhachiyan1991}.

We will use the following form of Gr\"unbaum's theorem. Recall that a
probability density $p$ on $\mathbb R^n$ is \emph{log-concave} if
\[
    p(\lambda x+(1-\lambda)y)
    \geq
    p(x)^\lambda p(y)^{1-\lambda}
\]
for every $x,y\in\mathbb R^n$ and $\lambda\in[0,1]$. Equivalently,
$\log p$ is concave on its support. In particular, the uniform distribution
on any convex body is log-concave.

\begin{theorem}[Gr\"unbaum's inequality]
\label{thm:grunbaum}
Let $p$ be a log-concave probability density on $\mathbb R^n$, and let
$c$ denote its barycenter. Then every halfspace $H$ containing $c$ satisfies
\[
    \Pr_{X\sim p}[X\in H]\geq \frac1e.
\]
\end{theorem}

For the uniform distribution on a convex body, this is the classical
inequality of Gr\"unbaum \cite{Grunbaum1960}. The formulation above for
general log-concave densities follows, for example, from
\cite[Lemma~5.12]{LovaszVempala2007}.

When all counterexamples are truthful, one may simply maintain the convex
body of points consistent with the feedback and query the ranking induced by
its centroid. Gr\"unbaum's theorem then gives constant-factor progress on
every round; this is the argument discussed in Section~\ref{sec:proof-overview}
and is closely related to the classical Kahn--Saks approach
\cite{KahnSaks1984}. We now turn directly to the additional ingredient needed
in the presence of untruthful counterexamples.

\paragraph{The weighted update.}

Fix a sufficiently large universal constant $\gamma>0$. We maintain a
nonnegative density $w_t$ on $[0,1]^n$. Initially,
\[
    w_0(x)=1.
\]
Let
\[
    W_t:=\int_{[0,1]^n}w_t(x)\,dx
    \qquad\text{and}\qquad
    p_t(x):=\frac{w_t(x)}{W_t}.
\]
On round $t+1$, let
\[
    c_t:=\mathbb E_{X\sim p_t}[X]
\]
be the barycenter of $p_t$, and query the ranking obtained by sorting the
coordinates of $c_t$, breaking ties arbitrarily.

Suppose that the returned counterexample is $\{i,j\}$ and that the queried
ranking places $i$ before $j$. Thus
\[
    (c_t)_i\leq(c_t)_j.
\]
Define
\[
    z_t(x):=x_j-x_i
\]
and the surrogate loss
\[
    \ell_t(x):=
    \bigl(1+\gamma n z_t(x)\bigr)_+,
    \qquad
    (a)_+:=\max\{a,0\}.
\]
We update
\[
    w_{t+1}(x)
    :=
    w_t(x)e^{-\ell_t(x)}.
    \tag{1}\label{eq:weight-update}
\]
Importantly, the learner does not need to know the lie budget $k$ in order
to perform this update.

\paragraph{Log-concavity and decrease of the total mass.}

We first show that the update retains the geometric property needed to apply
Gr\"unbaum's theorem.

\begin{lemma}
\label{lem:logconcave}
For every $t$, the density $w_t$ is log-concave on $[0,1]^n$.
\end{lemma}

\begin{proof}
The initial density is uniform on the convex set $[0,1]^n$, and is therefore
log-concave. Moreover, $z_t$ is affine, so
\[
    x\longmapsto (1+\gamma n z_t(x))_+
\]
is convex. Hence $-\ell_t$ is concave, and multiplying a log-concave
density by $e^{-\ell_t}$ preserves log-concavity.
\end{proof}

The next lemma is the analogue of the usual weighted-majority progress
bound.

\begin{lemma}
\label{lem:total-mass}
There is a universal constant $\rho<1$ such that
\[
    W_{t+1}\leq \rho W_t
\]
on every round.
\end{lemma}

\begin{proof}
Let
\[
    H_t:=\{x:z_t(x)\geq0\}
        =\{x:x_i\leq x_j\}.
\]
By construction, $c_t\in H_t$. Since $p_t$ is log-concave,
Theorem~\ref{thm:grunbaum} gives
\[
    p_t(H_t)\geq\frac1e.
\]
For every $x\in H_t$, we have $\ell_t(x)\geq1$, and hence
\[
    e^{-\ell_t(x)}\leq e^{-1}.
\]
Consequently,
\begin{align*}
    \frac{W_{t+1}}{W_t}
    &=
    \mathbb E_{X\sim p_t}
        \left[e^{-\ell_t(X)}\right] \\
    &\leq
    p_t(H_t)e^{-1}+
    \bigl(1-p_t(H_t)\bigr)\\
    &\leq
    1-\frac{1-e^{-1}}{e}.
\end{align*}
Thus the claim holds with
\[
    \rho:=1-\frac{1-e^{-1}}{e}<1.
\]
\end{proof}

It follows that after $T$ counterexamples,
\[
    W_T\leq\rho^T.
    \tag{2}\label{eq:total-mass-upper}
\]

\paragraph{The mass of the target cell.}

Let $\pi^\star$ be the target ranking and let
\[
    \mathcal C^\star:=\mathcal C_{\pi^\star}
\]
be its cell. Define
\[
    M_T:=\int_{\mathcal C^\star}w_T(x)\,dx.
\]
The key point is that, although the density becomes non-uniform inside
$\mathcal C^\star$, its total mass cannot decrease too quickly.

Let $Q$ denote the uniform distribution on $\mathcal C^\star$. Since
$\vol(\mathcal C^\star)=1/n!$, iterating
\eqref{eq:weight-update} gives
\begin{align}
    M_T
    &=
    \frac1{n!}
    \mathbb E_{X\sim Q}
    \left[
        \exp\left(
            -\sum_{t=0}^{T-1}\ell_t(X)
        \right)
    \right] \notag\\
    &\geq
    \frac1{n!}
    \exp\left(
        -\sum_{t=0}^{T-1}
        \mathbb E_{X\sim Q}[\ell_t(X)]
    \right),
    \tag{3}\label{eq:jensen-target}
\end{align}
where the inequality follows from Jensen's inequality.

We next bound the expected loss in a truthful and in an untruthful round.

\begin{lemma}
\label{lem:target-loss}
Let $\gamma>0$ be the universal constant fixed in the weighted update above.
For $X$ uniform on $\mathcal C^\star$:
\begin{enumerate}[label=(\roman*)]
    \item on a truthful round,
    \[
        \mathbb E[\ell_t(X)]\leq\frac1\gamma;
    \]
    \item on an untruthful round,
    \[
        \mathbb E[\ell_t(X)]\leq 1+\gamma n.
    \]
\end{enumerate}
\end{lemma}
\begin{proof}
Suppose, as above, that the query places $i$ before $j$.

Consider first a truthful round. The target then places $j$ before $i$.
Thus, throughout $\mathcal C^\star$,
\[
    X_j<X_i.
\]
Write
\[
    D:=X_i-X_j>0.
\]
Then
\[
    \ell_t(X)=(1-\gamma nD)_+.
    \tag{4}\label{eq:truthful-loss}
\]

A uniformly random point in any ranking cell has the same distribution,
after permuting its coordinates, as the order statistics of $n$ independent
uniform random variables on $[0,1]$. More explicitly, if
$U_1,\ldots,U_n$ are independent and uniform on $[0,1]$, and
\[
    U_{(1)}<U_{(2)}<\cdots<U_{(n)}
\]
are these random variables arranged in increasing order, then
$(U_{(1)},\ldots,U_{(n)})$ is uniformly distributed over the cell
\[
    \{0<x_1<\cdots<x_n<1\}.
\]

We use a standard fact about the spacings between uniform order statistics
\cite{Pyke1965}. The $n+1$ gaps
\[
    U_{(1)},\;
    U_{(2)}-U_{(1)},\ldots,
    U_{(n)}-U_{(n-1)},\;
    1-U_{(n)}
\]
are jointly uniform over the simplex of nonnegative vectors whose
coordinates sum to one (equivalently, they have the
$\operatorname{Dirichlet}(1,\ldots,1)$ distribution). In particular, every
individual gap $G$ satisfies
\[
    \Pr(G\leq s)
    =
    1-(1-s)^n
    \leq ns,
    \qquad 0\leq s\leq1.
    \tag{5}\label{eq:spacing}
\]

The difference $D=X_i-X_j$ is a sum of one or more consecutive spacings.
Choose any one spacing $G$ appearing in this sum. Since $D\geq G$,
\eqref{eq:truthful-loss} and \eqref{eq:spacing} give
\begin{align*}
    \mathbb E[\ell_t(X)]
    &\leq
    \Pr\left(D\leq\frac1{\gamma n}\right)\\
    &\leq
    \Pr\left(G\leq\frac1{\gamma n}\right)
    \leq\frac1\gamma.
\end{align*}

On an untruthful round, the target agrees with the query on $\{i,j\}$, and
hence $0<X_j-X_i<1$ throughout $\mathcal C^\star$. Thus,
\[
    \ell_t(X)
    =
    1+\gamma n(X_j-X_i)
    \leq 1+\gamma n.
\]
\end{proof}

Suppose now that among the first $T$ returned counterexamples at most $k$
are untruthful. Lemma~\ref{lem:target-loss} and
\eqref{eq:jensen-target} give
\[
    M_T
    \geq
    \frac1{n!}
    \exp\left(
        -\frac{T}{\gamma}
        -k(1+\gamma n)
    \right).
    \tag{6}\label{eq:target-mass-lower}
\]

\paragraph{Completing the proof.}

Since the target cell is contained in the cube,
\[
    M_T\leq W_T.
\]
Combining \eqref{eq:total-mass-upper} and
\eqref{eq:target-mass-lower},
\[
    \frac1{n!}
    \exp\left(
        -\frac{T}{\gamma}
        -k(1+\gamma n)
    \right)
    \leq
    \rho^T.
\]
Let
\[
    \alpha:=-\log\rho>0.
\]
Taking logarithms and rearranging gives
\[
    \left(\alpha-\frac1\gamma\right)T
    \leq
    \log(n!)+k(1+\gamma n).
\]
Choose the universal constant $\gamma$ sufficiently large that
$1/\gamma\leq\alpha/2$. Then
\[
    T
    =
    O\bigl(\log(n!)+nk\bigr)
    =
    O(n\log n+nk).
\]

Thus the oracle cannot return counterexamples for more than
$O(n\log n+nk)$ rounds; by that time the learner must have proposed the
target ranking. The algorithm is deterministic, proper, and its definition
does not use $k$. This proves the upper bound.

\subsection{The Lower Bound}
\label{sec:arbitrary-lower}

We now prove the matching lower bound. The two terms have different
origins. The $\Omega(n\log n)$ term is already necessary when all
counterexamples are truthful and follows from a sorting argument. The
$\Omega(nk)$ term captures the additional cost of untruthful
counterexamples and is based on a Condorcet cycle.

\begin{theorem}
\label{thm:arbitrary-lower}
For every $n\geq 3$ and $k\geq 0$, every possibly randomized proper learner
for arbitrary rankings has worst-case expected query complexity
\[
    \Omega(n\log n+nk).
\]
\end{theorem}

\begin{proof}
We prove the two terms separately.

We first establish the $\Omega(n\log n)$ lower bound using only truthful
counterexamples. For simplicity assume that $n$ is a power of two; for
general $n$, we may restrict attention to the largest power of two below
$n$. Consider the comparison tree induced by MergeSort. Every
root-to-leaf path in this tree has length
\[
    D=\Omega(n\log n).
\]
Indeed, at every level of the recursion, merging pairs of equal-size lists
requires altogether $\Omega(n)$ comparisons, and there are
$\Theta(\log n)$ levels.

Choose a random root-to-leaf branch of this tree by choosing each of the two
children with probability $1/2$ at every internal node. The resulting leaf
determines a target ranking $\pi^\star$, which is thus fixed before the
interaction begins.

We describe a truthful oracle for this target. The oracle maintains a
current node on the chosen branch, initially the root. Suppose that the
learner proposes a ranking $\pi$. If $\pi$ disagrees with the outcome of
some comparison on the portion of the branch already traversed, the oracle
may simply return such a comparison as a truthful counterexample.

Otherwise, starting from the current node, the oracle follows the chosen
branch down the comparison tree. As long as $\pi$ agrees with the outcome
of the comparison at the current node, the oracle proceeds to the next
node without returning anything. At the first comparison on which $\pi$
disagrees with the chosen branch, the oracle returns that pair as a
counterexample and makes this node part of the traversed portion of the
branch. If no such comparison is encountered before reaching the leaf,
then $\pi$ agrees with every comparison on the root-to-leaf path and hence
$\pi=\pi^\star$; the oracle then declares the ranking correct.

All counterexamples produced by this strategy are truthful. We now bound
how quickly a learner can move down the random branch. Let $A_t$ denote the
number of previously untraversed edges of the comparison tree that are traversed on
the $t$-th query. If the learner contradicts an earlier comparison then~$A_t=0$. Otherwise, conditional on everything that has happened so far and
on the learner's current query, every new branch choice encountered by the
oracle is still an independent fair choice between the two children. At
each such node, the learner's ranking prescribes one of the two possible
outcomes. Hence it agrees with the random branch with probability $1/2$,
and disagrees with probability $1/2$.

Consequently, $A_t$ is dominated by a geometric random variable with mean
$2$, and therefore
\[
    \mathbb E[A_t\mid\text{history before query }t]\leq 2.
\]
This remains true when the learner is randomized, since we may condition
also on its internal randomness up to this point.

Let $Q$ be the number of queries made before the target is identified. By
the time the interaction terminates, the entire root-to-leaf path has been
traversed, so
\[
    \sum_{t=1}^{Q} A_t\geq D.
\]
Therefore
\begin{align*}
    D
    &\leq
    \mathbb E\left[\sum_{t=1}^{Q}A_t\right] \\
    &=
    \sum_{t\geq1}
    \mathbb E\left[
        \mathbf 1_{\{Q\geq t\}}A_t
    \right] \\
    &\leq
    2\sum_{t\geq1}\Pr(Q\geq t)
    =
    2\mathbb E[Q].
\end{align*}
Thus
\[
    \mathbb E[Q]\geq D/2=\Omega(n\log n),
\]
where the expectation is over both the random branch and the learner's
randomness. Averaging over the choice of the branch, there must therefore
exist a fixed branch, and hence a fixed target ranking $\pi^\star$, for
which the learner requires $\Omega(n\log n)$ queries in expectation.

We next establish the $\Omega(nk)$ term. Consider the directed cycle
\[
    1\to2\to\cdots\to n\to1,
\]
and denote its edges by $e_1,\ldots,e_n$. For each $r\in[n]$, let
$\sigma_r$ be the linear order obtained by deleting $e_r$ from the cycle
and taking the resulting directed path as the order. Thus $\sigma_r$ agrees
with every cycle edge except $e_r$.

Every linear order $\pi$ violates at least one edge of the cycle. Moreover,
if $\pi$ violates $e_r$ and the oracle returns this pair, then the
counterexample is truthful with respect to every $\sigma_s$, $s\neq r$,
and untruthful with respect to $\sigma_r$.

Define a reference interaction that does not depend on the target. Whenever
the learner proposes a ranking $\pi$, choose, according to any fixed rule,
one cycle edge violated by $\pi$ and return it. For example, choose the
violated edge of smallest index. This produces an infinite sequence of
edges
\[
    e_{r_1},e_{r_2},\ldots
\]
that depends on the learner's randomness, but not on the choice of the
target.

For each $r\in[n]$, let $\tau_r$ be the round on which $e_r$ is returned
for the $(k+1)$-st time in this reference interaction, with
$\tau_r=\infty$ if this never occurs. Let
\[
    \tau_{(1)}\leq\tau_{(2)}\leq\cdots\leq\tau_{(n)}
\]
be these values arranged in increasing order. Then
\[
    \tau_{(j)}\geq j(k+1),
    \qquad j=1,\ldots,n.
    \tag{21}\label{eq:cycle-hitting-times}
\]
Indeed, by round $\tau_{(j)}$, at least $j$ distinct cycle edges have each
been returned at least $k+1$ times, so at least $j(k+1)$ counterexamples
must have been returned in total. Consequently,
\[
    \frac1n\sum_{r=1}^n \tau_r
    \geq
    \frac1n\sum_{j=1}^n j(k+1)
    =
    \frac{(n+1)(k+1)}2.
    \tag{22}\label{eq:average-cycle-hitting-time}
\]
If some $\tau_r=\infty$, this inequality is immediate.

Now choose $R$ uniformly from $[n]$, independently of the learner, and set
\[
    \pi^\star=\sigma_R.
\]
Thus the target is fixed before the interaction begins.

We define an oracle for this target by following the reference interaction
until round $\tau_R$. Before this round, the edge $e_R$ has been returned
at most $k$ times. Each occurrence of $e_R$ is an untruthful counterexample
with respect to $\sigma_R$, while every other returned cycle edge is
truthful. Hence all these responses are valid and use at most $k$
untruthful counterexamples.

On round $\tau_R$, the reference rule is about to return $e_R$ for the
$(k+1)$-st time. If the learner's query is $\sigma_R$, the oracle declares
it correct. Otherwise, the query violates some cycle edge other than
$e_R$, and the oracle returns such an edge instead; this counterexample is
truthful. Indeed, a ranking whose only violated cycle edge is $e_R$ must be
exactly $\sigma_R$. From this point on, whenever the learner proposes a
ranking different from $\sigma_R$, the oracle may return any truthful
counterexample, and it declares the ranking correct when $\sigma_R$ is
proposed. Thus this defines a valid oracle using at most $k$ untruthful
counterexamples.

Let $Q$ denote the number of queries before the interaction terminates.
Under the coupling above, the learner cannot terminate before round
$\tau_R$, and therefore
\[
    Q\geq\tau_R.
\]
Taking expectation over both the random choice of $R$ and the learner's
internal randomness, and using \eqref{eq:average-cycle-hitting-time}, gives
\[
\begin{aligned}
    \mathbb E[Q]
    &\geq \mathbb E[\tau_R] \\
    &= \mathbb E\left[\frac1n\sum_{r=1}^n\tau_r\right] \\
    &\geq \frac{(n+1)(k+1)}2
    = \Omega(nk).
\end{aligned}
\]
Averaging over $R$, there is therefore some fixed target $\sigma_R$ for
which the expected number of queries, over the learner's randomness, is
$\Omega(nk)$.

We have proved that the worst-case expected query complexity of every
randomized proper learner is both $\Omega(n\log n)$ and $\Omega(nk)$.
Consequently,
\[
    \max\{\Omega(n\log n),\Omega(nk)\}
    =
    \Omega(n\log n+nk),
\]
up to a universal constant.
\end{proof}

\section{Geometric Rankings}
\label{sec:geometric-rankings}

We now turn to ranking classes with a low-dimensional geometric
representation. Let $\mathcal H$ be a class of rankings of $[n]$ for which
there exist vectors
\[
    v_1,\ldots,v_n\in\mathbb R^d
\]
such that every $\pi\in\mathcal H$ is induced by some direction
$w\in\mathbb R^d$: namely,
\[
    i\prec_\pi j
    \qquad\Longleftrightarrow\qquad
    \langle v_i,w\rangle<\langle v_j,w\rangle.
\]
We consider only directions that induce strict linear orders.

\subsection{The Upper Bound}
\label{sec:geometric-upper}

Let $\mathcal H$ be a $d$-dimensional geometric class of rankings over
$[n]$, represented by vectors
\[
    v_1,\ldots,v_n\in\mathbb R^d.
\]
Thus every $\pi\in\mathcal H$ is induced by some direction
$w\in\mathbb R^d$, in the sense that
\[
    i\prec_\pi j
    \qquad\Longleftrightarrow\qquad
    \langle v_i,w\rangle<\langle v_j,w\rangle.
\]
We consider only directions that induce strict linear orders.

The algorithm is based on weighted majority. The key point is that, although
the pairwise weighted-majority predictions may not themselves form a linear
order, sufficiently strong majorities are always acyclic in a
$d$-dimensional geometric class. We can therefore complete them to a genuine
linear order. This queried order need not itself belong to $\mathcal H$.

\begin{theorem}
\label{thm:geometric-upper}
Every $d$-dimensional geometric class $\mathcal H$ of rankings over $n$
items can be learned using
\[
    O(d^2\log n+dk)
\]
queries in the presence of at most $k$ untruthful counterexamples. Every
query made by the learner is a linear order, but need not belong to
$\mathcal H$. The learner is deterministic and does not need to know~$k$.
\end{theorem}

\begin{proof}
We maintain a weight $q_t(\pi)$ for every ranking $\pi\in\mathcal H$.
Initially,
\[
    q_0(\pi)=1
    \qquad\text{for every }\pi\in\mathcal H.
\]
Let
\[
    Q_t:=\sum_{\pi\in\mathcal H}q_t(\pi),
\]
and let $P_t$ be the probability distribution on $\mathcal H$ obtained by
normalizing these weights:
\[
    P_t(\pi):=\frac{q_t(\pi)}{Q_t}.
    \tag{13}\label{eq:geometric-normalized-weights}
\]
Thus, when we write
\[
    \Pi_t\sim P_t,
\]
we mean that $\Pi_t$ is a random ranking from $\mathcal H$, sampled
according to the current normalized weights.

For two distinct items $i,j$, define
\[
    p_t(i,j)
    :=
    \Pr_{\Pi_t\sim P_t}
    \bigl[i\prec_{\Pi_t}j\bigr].
    \tag{14}\label{eq:geometric-pairwise-weight}
\]
Since every ranking is a linear order,
\[
    p_t(i,j)+p_t(j,i)=1.
\]

We form a directed graph $G_t$ on the $n$ items by putting an edge
$i\to j$ whenever
\[
    p_t(i,j)>1-\frac{1}{d+1}.
    \tag{15}\label{eq:strong-majority}
\]
Thus $G_t$ contains the comparisons supported by an overwhelming majority
of the current weight. The main geometric observation is that this relation
cannot contain a directed cycle.

\begin{lemma}
\label{lem:strong-majority-acyclic}
The directed graph $G_t$ is acyclic.
\end{lemma}

\begin{proof}
Suppose, towards a contradiction, that $G_t$ contains a directed cycle
\[
    i_1\to i_2\to\cdots\to i_m\to i_1.
\]
For each edge of the cycle, define
\[
    a_r:=v_{i_{r+1}}-v_{i_r},
    \qquad r=1,\ldots,m,
\]
where $i_{m+1}:=i_1$. These vectors telescope, and hence
\[
    \sum_{r=1}^m a_r=0.
\]
In particular,
\[
    0\in\operatorname{conv}\{a_1,\ldots,a_m\}.
\]

By Carath\'eodory's theorem \cite{Matousek2002}, there is a set
$R\subseteq[m]$ with
\[
    |R|\leq d+1
\]
and coefficients $\lambda_r>0$, $r\in R$, such that
\[
    \sum_{r\in R}\lambda_r=1
    \qquad\text{and}\qquad
    \sum_{r\in R}\lambda_r a_r=0.
    \tag{16}\label{eq:caratheodory-dependence}
\]
Here we have simply discarded the indices with zero coefficient.

Now sample a ranking
\[
    \Pi_t\sim P_t
\]
according to the current normalized weights. For every $r\in R$, the edge
$i_r\to i_{r+1}$ belongs to $G_t$, so by
\eqref{eq:strong-majority},
\[
    \Pr\bigl[i_r\prec_{\Pi_t}i_{r+1}\bigr]
    >
    1-\frac1{d+1}.
\]
Equivalently,
\[
    \Pr\bigl[i_{r+1}\prec_{\Pi_t}i_r\bigr]
    <
    \frac1{d+1}.
\]
Since $|R|\leq d+1$, the union bound gives
\begin{align*}
    &\Pr\bigl[
        i_r\prec_{\Pi_t}i_{r+1}
        \text{ for every }r\in R
    \bigr] \\
    &\qquad\geq
    1-
    \sum_{r\in R}
    \Pr\bigl[i_{r+1}\prec_{\Pi_t}i_r\bigr]
    >0.
\end{align*}
Therefore there exists at least one ranking $\pi\in\mathcal H$ satisfying
all the comparisons
\[
    i_r\prec_\pi i_{r+1},
    \qquad r\in R.
\]

Since $\pi\in\mathcal H$, it is induced by some direction
$w\in\mathbb R^d$. Hence, for every $r\in R$,
\[
    \langle v_{i_r},w\rangle
    <
    \langle v_{i_{r+1}},w\rangle,
\]
and therefore
\[
    \langle a_r,w\rangle>0.
\]
Multiplying by $\lambda_r>0$ and summing gives
\[
    \left\langle
        \sum_{r\in R}\lambda_r a_r,w
    \right\rangle
    >0,
\]
contradicting \eqref{eq:caratheodory-dependence}. Thus $G_t$ is acyclic.
\end{proof}

Since $G_t$ is acyclic, it has a linear extension. The learner chooses any
such extension $\widehat\pi_t$ and queries it. Notice that
$\widehat\pi_t$ is a genuine linear order, but need not belong to
$\mathcal H$.

The choice of threshold in \eqref{eq:strong-majority} ensures that every
comparison made by $\widehat\pi_t$ still has noticeable support under the
current weighted distribution.

\begin{lemma}
\label{lem:geometric-support}
If $\widehat\pi_t$ places $i$ before $j$, then
\[
    p_t(i,j)\geq\frac1{d+1}.
\]
\end{lemma}

\begin{proof}
Suppose instead that
\[
    p_t(i,j)<\frac1{d+1}.
\]
Then
\[
    p_t(j,i)
    =
    1-p_t(i,j)
    >
    1-\frac1{d+1},
\]
so $G_t$ contains the edge $j\to i$. This contradicts the fact that
$\widehat\pi_t$ is a linear extension of $G_t$.
\end{proof}

Suppose now that the oracle returns the pair $\{i,j\}$, and orient the
notation so that $\widehat\pi_t$ places $i$ before $j$. We perform the
standard weighted-majority update \cite{LittlestoneWarmuth1994}:
\[
    q_{t+1}(\pi)
    :=
    \begin{cases}
        \frac12 q_t(\pi),
        & \text{if }i\prec_\pi j,\\[3pt]
        q_t(\pi),
        & \text{if }j\prec_\pi i.
    \end{cases}
    \tag{17}\label{eq:geometric-weight-update}
\]
Thus we halve the weight of precisely those rankings in $\mathcal H$ that
agree with the learner's proposed order on the returned pair.

By Lemma~\ref{lem:geometric-support}, these rankings carry at least a
$1/(d+1)$ fraction of the total weight. Therefore
\begin{align*}
    Q_{t+1}
    &\leq
    \left(1-\frac1{d+1}\right)Q_t
    +
    \frac12\frac1{d+1}Q_t\\
    &=
    \left(
        1-\frac{1}{2(d+1)}
    \right)Q_t.
\end{align*}
Hence
\[
    Q_{t+1}
    \leq
    \left(
        1-\frac{1}{2(d+1)}
    \right)Q_t.
    \tag{18}\label{eq:geometric-total-weight}
\]
This decrease occurs on every round, regardless of whether the returned
counterexample is truthful.

Consider now the target ranking $\pi^\star\in\mathcal H$. On a truthful
round, $\pi^\star$ disagrees with the learner on the returned pair, so its
weight is unchanged. On an untruthful round, $\pi^\star$ agrees with the
learner on that pair, so its weight is halved. Consequently, after $T$
rounds containing at most $k$ untruthful counterexamples,
\[
    q_T(\pi^\star)\geq 2^{-k}.
    \tag{19}\label{eq:geometric-target-weight}
\]

It remains to bound the initial total weight
\[
    Q_0=|\mathcal H|.
\]
For every pair $\{i,j\}$, its relative order under a direction $w$ is
determined by the sign of
\[
    \langle v_i-v_j,w\rangle.
\]
Thus every ranking in $\mathcal H$ corresponds to a full-dimensional cell
in the arrangement of the
\[
    N:=\binom n2
\]
hyperplanes
\[
    \langle v_i-v_j,w\rangle=0,
    \qquad 1\leq i<j\leq n.
\]
All these hyperplanes pass through the origin. A central arrangement of $N$ hyperplanes in $\mathbb R^d$ has at most
\[
    2\sum_{r=0}^{d-1}\binom{N-1}{r}
\]
full-dimensional cells \cite{Matousek2002}. Since here
$N=\binom n2$, we obtain
\[
    |\mathcal H|
    \leq
    2\sum_{r=0}^{d-1}
        \binom{\binom n2-1}{r}
    = n^{O(d)},
\]
and hence
\[
    \log|\mathcal H|=O(d\log n).
    \tag{20}\label{eq:number-geometric-rankings}
\]
Finally, since
\[
    q_T(\pi^\star)\leq Q_T,
\]
iterating \eqref{eq:geometric-total-weight} and using
\eqref{eq:geometric-target-weight} gives
\[
    2^{-k}
    \leq
    Q_T
    \leq
    |\mathcal H|
    \left(
        1-\frac{1}{2(d+1)}
    \right)^T.
\]
Taking logarithms yields
\[
    T\,
    \left(
        -\log\left(
            1-\frac{1}{2(d+1)}
        \right)
    \right)
    \leq
    \log|\mathcal H|+k\log 2.
\]
Using
\[
    -\log(1-x)\geq x
    \qquad (0<x<1),
\]
we conclude that
\[
    T
    \leq
    2(d+1)
    \bigl(
        \log|\mathcal H|+k\log2
    \bigr).
\]
Together with \eqref{eq:number-geometric-rankings}, this gives
\[
    T
    =
    O(d^2\log n+dk).
\]
Since the learner does not use $k$, this completes the proof.
\end{proof}

\paragraph{A centerpoint alternative and geometric realizability.}
The same upper bound can alternatively be obtained from the weighted
centerpoint theorem \cite{Matousek2002}. This approach is closely analogous
to the Gr\"unbaum-based algorithm for unrestricted rankings. For every
$\pi\in\mathcal H$, choose a unit direction
$w_\pi\in\mathbb R^d$ that realizes $\pi$, and place mass $q_t(\pi)$ at
$w_\pi$. Let $c_t$ be a weighted centerpoint of this discrete measure, and query the
ranking induced by a generic direction sufficiently close to $c_t$, chosen
only to break possible ties among the projections at $c_t$.

For every pair oriented by this order, the corresponding closed halfspace
of directions contains $c_t$. The centerpoint theorem therefore implies
that this halfspace contains at least a $1/(d+1)$ fraction of the current
weight. Consequently, if the oracle returns this pair, halving the weights
of the rankings that agree with the query decreases the total weight by a
factor of at most
\[
    1-\frac{1}{2(d+1)}.
\]
The target weight is halved only on untruthful rounds. Since
$\log|\mathcal H|=O(d\log n)$, the usual weighted-majority calculation again
gives
\[
    O(d^2\log n+dk)
\]
queries.

This alternative has the additional feature that every query is induced by
a direction in the given geometric representation. Thus, if
$\mathcal H$ consists of all rankings induced by this representation, the
learner is proper in the usual sense. Under our more general convention,
where $\mathcal H$ may be an arbitrary subclass of the rankings induced by
the representation, the queried ranking need not itself belong to
$\mathcal H$, but it is always geometrically realizable.

This property may also be useful from a practical perspective. In the
restaurant example from the introduction, a geometrically realizable
ranking is specified by only $d$ coefficients, one for each feature. Such
a representation may be more convenient to evaluate at prediction time
and can also offer a degree of interpretability, since the coefficients
indicate the relative importance assigned to the different features.

The parallel with the unrestricted algorithm is that there the learner
queries the ranking induced by the barycenter of a log-concave measure and
uses Gr\"unbaum's theorem to control the mass of every relevant halfspace,
whereas here the barycenter is replaced by a centerpoint of the discrete
weighted measure on realizing directions.

\subsection{The Lower Bound}
\label{sec:geometric-lower}

We now prove the lower bound for geometric ranking classes. In fact, a
single natural class---rankings induced by low-degree polynomials---already
gives all the terms in the lower bound.

\begin{theorem}
\label{thm:geometric-lower}
For every $2\leq d\leq n$, there is a $d$-dimensional geometric class of
rankings over $n$ items for which every possibly randomized learner whose
queries are linear orders has worst-case expected query complexity
\[
    \Omega(d\log n+dk)
\]
in the presence of at most $k$ untruthful counterexamples.
\end{theorem}

\begin{proof}
Identify the $n$ items with the points $1,\ldots,n$ on the real line, and
let $\mathcal P_{n,d}$ be the class of rankings obtained by ordering these
points according to the values of a polynomial $p$ of degree at most $d$.
We restrict to polynomials for which
\[
    p(1),\ldots,p(n)
\]
are all distinct.

This is a $d$-dimensional geometric class. Indeed, writing
\[
    p(x)=a_0+a_1x+\cdots+a_dx^d,
\]
the constant term does not affect the ranking, and
\[
    p(i)
    =
    a_0+
    \left\langle
        (i,i^2,\ldots,i^d),
        (a_1,\ldots,a_d)
    \right\rangle.
\]
Thus $\mathcal P_{n,d}$ is represented by the points
\[
    v_i=(i,i^2,\ldots,i^d)\in\mathbb R^d
\]
on the moment curve.

We first observe that $\mathcal P_{n,d}$ already inherits the unrestricted
lower bounds on $d$ items. Fix any set $S$ of $d$ items. Every linear order
on $S$ can be realized by a polynomial of degree at most $d-1$: assign
distinct real values to the points of $S$ in the desired order and
interpolate a polynomial through them. 

Consequently, the two adversaries from
Theorem~\ref{thm:arbitrary-lower} may be applied using only comparisons
between items in $S$. The resulting target ranking can always be chosen
from $\mathcal P_{n,d}$. Hence every learner for $\mathcal P_{n,d}$ has
worst-case expected query complexity at least
\[
    \Omega(d\log d+dk).
    \tag{23}\label{eq:geometric-lower-first}
\]
It remains to obtain the dependence on $n$ in the noiseless term. We show
that, when $n$ is sufficiently larger than $d$,
\[
    \Omega\bigl(d\log(n/d)\bigr)
    \tag{24}\label{eq:geometric-lower-second}
\]
truthful counterexamples may be necessary.

Let
\[
    m:=d-1.
\]
Choose a power of two
\[
    M=\Theta(n/d),
\]
and choose $m$ disjoint blocks
$B_1,\ldots,B_m$ of consecutive items, ordered from left to right, each
containing $\Theta(M)$ items. Within each block $B_r$, choose $M$ possible
locations for a root of the derivative, separated by a constant distance.
For example, we may take
\[
    a_{r,s}=b_r+2s+\frac12,
    \qquad s=0,\ldots,M-1,
\]
where the integers $b_r$ are chosen so that all these points lie in $B_r$.

For every
\[
    \mathbf{s}=(s_1,\ldots,s_m)\in\{0,\ldots,M-1\}^m,
\]
define
\[
    q_{\mathbf{s}}(x)
    :=
    \prod_{r=1}^m (x-a_{r,s_r}),
\]
and let $p_{\mathbf{s}}$ be an antiderivative of $q_{\mathbf{s}}$. Thus
\[
    p'_{\mathbf{s}}=q_{\mathbf{s}},
\]
and $p_{\mathbf{s}}$ has degree $m+1=d$. Its derivative has exactly one
simple zero in each block $B_r$, namely $a_{r,s_r}$. Equivalently,
$p_{\mathbf{s}}$ has one turning point in each block, whose location can be
chosen among $M$ possibilities.

We now construct a balanced comparison tree that determines these roots one
at a time. Consider a block $B_r$, and suppose that the possible values of
$s_r$ currently form a consecutive interval. Split this interval into two
equal halves, say between $q$ and $q+1$. By the choice of the candidate
locations, there are two consecutive items $j,j+1$ lying strictly between
$a_{r,q}$ and $a_{r,q+1}$.

For $x\in[j,j+1]$, all factors
\[
    x-a_{\ell,s_\ell},
    \qquad \ell\neq r,
\]
have signs that are independent of the choices of
$s_1,\ldots,s_m$: the corresponding roots lie in blocks entirely to the
left or entirely to the right of $B_r$. On the other hand, the sign of
\[
    x-a_{r,s_r}
\]
depends on which half contains $s_r$. If $s_r\leq q$, then
$a_{r,s_r}<j$, whereas if $s_r\geq q+1$, then
$a_{r,s_r}>j+1$. Consequently, the sign of
$p'_{\mathbf{s}}(x)$ throughout $[j,j+1]$ is opposite in the two cases.
Therefore the sign of
\(p_{\mathbf{s}}(j+1)-p_{\mathbf{s}}(j)\)
is also opposite. In other words, the relative order of the two consecutive
items $j$ and $j+1$ determines which half contains the root of
$p'_{\mathbf{s}}$ in $B_r$.

Thus the location of the root in each block can be determined by a balanced
binary search of depth $\log_2 M$. Performing these searches successively
for the $m=d-1$ blocks gives a comparison tree in which every
root-to-leaf path has depth
\[
    (d-1)\log_2 M
    =
    \Omega\bigl(d\log(n/d)\bigr).
\]

Each leaf specifies one root location in every block and hence a degree-$d$
polynomial $p_{\mathbf{s}}$ consistent with every comparison along the
corresponding path. 

Finally, apply the same random-branch argument as in the sorting lower
bound. Choose the hidden branch of this comparison tree by choosing each
child independently and uniformly, and fix a polynomial ranking associated
with the resulting leaf as the target. On each query, the oracle follows
the hidden branch until reaching the first comparison on which the
learner's proposed order disagrees, and returns this pair as a truthful
counterexample.

Conditional on the entire interaction so far, each newly encountered branch
is still a fair independent choice. Hence a single query traverses at most
two new edges of the hidden branch in expectation. Since the branch has
depth
\[
    \Omega\bigl(d\log(n/d)\bigr),
\]
the expected number of queries is
\[
    \Omega\bigl(d\log(n/d)\bigr).
\]
Averaging over the random choice of the target, there is therefore a fixed
polynomial ranking for which every randomized learner requires this many
queries in expectation. This proves
\eqref{eq:geometric-lower-second}.
\end{proof}

\section{Open Questions and Future Directions}
\label{sec:open-questions}

Our results leave several natural questions open.

\paragraph{Geometric ranking classes.}
Perhaps the most immediate question concerns geometric ranking classes. For a $d$-dimensional geometric class over $n$ items, we prove an upper bound of
\[
    O(d^2\log n+dk)
\]
and a lower bound of
\[
    \Omega(d\log n+dk).
\]
Thus the dependence on the number of untruthful counterexamples is tight,
while a factor of $d$ remains in the noiseless term. It would be interesting
to determine whether the upper bound can be improved to
$O(d\log n+dk)$, or whether there are geometric classes requiring
$\Omega(d^2\log n)$ queries.

\paragraph{Starting from partial information.}
Another natural extension is to assume that some comparisons are known in
advance. In the recommendation-system example, for instance, the user may
have already specified preferences between certain pairs of restaurants.
More abstractly, the known comparisons form a poset $P$, and the unknown
target ranking is one of its $L$ linear extensions.

In the truthful setting, the optimal deterministic query complexity in this
more general problem is~$\Theta(\log L)$. The upper bound follows from
essentially the same geometric argument used above, applied to the order
polytope of $P$. For the lower bound, recall the theorem of Kahn and Saks:
whenever $P$ is not already a total order, there is an incomparable pair
$i,j$ such that at least a $3/11$ fraction of the linear extensions place
$i$ before $j$, and at least a $3/11$ fraction place $j$ before $i$.
Applying this theorem recursively gives a balanced comparison tree for the
$L$ linear extensions: at every internal node we compare such a pair and
restrict to the linear extensions consistent with the chosen orientation.
Since either child retains at least a $3/11$ fraction of the current linear
extensions, every root-to-leaf path has depth $\Omega(\log L)$. The same
random-branch argument used in our lower bound for unrestricted rankings
then implies that $\Omega(\log L)$ queries are necessary.

The situation becomes less clear when untruthful counterexamples are
allowed. Let $x$ denote the width of $P$ (i.e.\ the largest possible size of an antichain in $P$). A maximum antichain of size $x$
gives, by applying our Condorcet-cycle construction to these items, the
lower bound
\[
    \Omega(\log L+xk).
\]
On the other hand, the strong-majority argument used above can be adapted
to give
\[
    O\bigl(x(\log L+k)\bigr).
\]
This upper bound is not always tight: when $P$ is an antichain, $x=n$ and
$L=n!$, whereas our main result gives the sharper bound
\[
    \Theta(\log L+xk)
    =
    \Theta(n\log n+nk).
\]
It would be interesting to characterize, up to constant factors, the optimal
query complexity for every starting poset $P$ and every $k$. In particular,
it is not clear whether this complexity is determined by $L$ and $x$ alone,
or whether additional structural parameters of the poset are needed.

\paragraph{Optimal constants for unrestricted rankings.}
Even for unrestricted rankings, where our bounds are tight up to constant
factors, the precise query complexity remains open. In particular, it would
be interesting to determine the optimal constants in front of the
$n\log n$ and $nk$ terms. The two terms arise from rather different
phenomena---sorting in the noiseless case and the Condorcet-cycle
construction in the noisy case---and it is not clear whether a single
optimal strategy can simultaneously achieve the best constants for both.

\paragraph{Computational efficiency.}
In this work we focus on query complexity and do not develop fully efficient
implementations of our algorithms. As discussed above, the log-concave
distributions arising in the unrestricted upper bound can in principle be
sampled using standard random-walk methods, allowing the relevant
barycenters to be approximated in randomized polynomial time. It would be
useful to make this implementation explicit and to understand its
computational complexity more precisely. More ambitiously, we would like to
find efficient combinatorial algorithms for these problems. Such algorithms
might also lead to more direct proofs of our upper bounds, avoiding the use
of Gr\"unbaum's theorem and perhaps shedding new light on the classical
balanced-pair phenomenon for linear extensions.

\paragraph{Richer noise models.}
Our noise model is deliberately simple: we only assume that the total number
of untruthful counterexamples is bounded. A natural next step is to study
richer stochastic noise models. For example, fix $p<1/2$ and suppose that
on each round, conditional on the entire history, the oracle may return an
arbitrary untruthful counterexample with probability at most $p$, while
otherwise its response must be truthful. It would be interesting to
determine the optimal query complexity in this model and whether our
weighted approach extends to it. One could also consider feedback generated
from an underlying preference model, such as Bradley--Terry
\cite{BradleyTerry1952}. The methods developed here may extend to some such
settings, but a systematic study remains open.

\subsection*{AI disclosure}
ChatGPT was used in several stages of the preparation of this paper.
First, GPT-5.6 Sol with Extra High reasoning was used for
editing and drafting the manuscript, as well as for checking references
and searching for additional related work. Second, GPT-5.6 Sol with Extra
High reasoning was used to assist with calculations and verification
surrounding the surrogate log-concave update used in the upper bound for
unrestricted rankings. Finally, GPT-5.6 Sol Pro was asked to investigate
the open questions suggested by this work. It made some progress on
improving constants and suggested several possible approaches, but did not
resolve any of these questions. All mathematical arguments and claims
appearing in the paper were found and checked by the authors and are their responsibility.

\bibliographystyle{alpha}
\bibliography{references}

\end{document}